\documentclass[10pt,conference]{IEEEtran}
\IEEEoverridecommandlockouts

\usepackage{cite}
\usepackage{amsthm, amsmath,amssymb,amsfonts}
\usepackage{algorithmic}
\usepackage{graphicx}
\usepackage{textcomp}
\usepackage[T1]{fontenc}
\usepackage{xcolor}
\usepackage{tabularx}
\usepackage{etoolbox}
\usepackage[utf8]{inputenc} 
\usepackage{url}            
\usepackage{booktabs}       
\usepackage{nicefrac}       
\usepackage{microtype}      
\usepackage{longtable}
\usepackage{subcaption}
\usepackage{float} 
\usepackage{hyperref}
\usepackage{url}
\usepackage{placeins}   
\usepackage{dblfloatfix} 
\usepackage{enumitem}
\usepackage{tikz}
\usepackage{multirow}
\usepackage{threeparttable}
\usetikzlibrary{positioning,calc,shapes.geometric,arrows.meta}
\usepackage[nolist]{acronym}

\newcommand{\doublemidrule}{%
  \specialrule{0.4pt}{0pt}{0pt}%
  \specialrule{0.4pt}{1.2pt}{0pt}
}

\begin{acronym}[XXX] 
\acro{ai}[AI]{Artificial Intelligence}
\acro{auc}[AUC]{Area Under the Curve}
\acro{cnn}[CNN]{Convolutional Neural Network}
\acro{dl}[DL]{Deep Learning}
\acro{dnn}[DNN]{Deep Neural Network}
\acro{dnns}[DNNs]{Deep Neural Networks}
\acro{gelu}[GELU]{Gaussian Error Linear Unit}
\acro{glai}[GLAI]{GreenLightningAI}
\acro{kd}[KD]{Knowledge Distillation}
\acro{ml}[ML]{Machine Learning}
\acro{mlp}[MLP]{Multilayer Perceptron}
\acro{mva}[MVA]{Maximum Validation Accuracy}
\acro{mha}[MHA]{Multihead Attention}
\acro{moe}[MoE]{Mixture-of-Experts}
\acro{mimick}[NM]{Neural Mimicking}
\acro{nlp}[NLP]{Natural Language Processing}
\acro{cv}[CV]{Computer Vision}
\acro{nn}[NN]{Neural Network}
\acro{oui}[OUI]{Overfitting-Underfitting Indicator}
\acro{prelu}[PReLU]{Parametric Rectified Linear Unit}
\acro{poc}[PoC]{Proof-of-Concept}
\acro{relu}[ReLU]{Rectified Linear Unit}
\acro{silu}[SiLU]{Sigmoid Linear Unit}
\acro{sgd}[SGD]{Stochastic Gradient Descent}
\acro{svd}[SVD]{Singular Value Decomposition}
\acro{sota}[SOTA]{state-of-the-art}
\acro{vt}[VT]{Vision Transformer}
\acro{vit}[ViT]{Vision Transformer}
\acro{wd}[WD]{Weight Decay}
\end{acronym}

\newtheorem{proposition}{Proposition}

\theoremstyle{definition}
\newtheorem{definition}{Definition}
\theoremstyle{remark}
\newtheorem{remark}{Remark}
\newtheorem{corollary}{Corollary}

\definecolor{c_train}{RGB}{31,119,180}   
\definecolor{c_val}{RGB}{255,127,14}     
\definecolor{c_hamm}{RGB}{44,160,44}     
\definecolor{c_wdist}{RGB}{214,39,40} 

\def\BibTeX{{\rm B\kern-.05em{\sc i\kern-.025em b}\kern-.08em
    T\kern-.1667em\lower.7ex\hbox{E}\kern-.125emX}}
\begin{document}

\title{Regime Change Hypothesis: Foundations for Decoupled Dynamics in Neural Network Training
}

\author{\IEEEauthorblockN{ Cristian Pérez-Corral}
\IEEEauthorblockA{\textit{Dept. of Computer Engineering} \\
\textit{Universitat Politècnica de València}\\
Valencia, Spain \\
cpercor@upv.es}
\and
\IEEEauthorblockN{ Alberto Fernández-Hernández}
\IEEEauthorblockA{\textit{Dept. of Computer Engineering} \\
\textit{Universitat Politècnica de València}\\
Valencia, Spain \\
a.fernandez@upv.es}
\and
\IEEEauthorblockN{Jose I. Mestre}
\IEEEauthorblockA{\textit{Dept. of Computer Engineering} \\
\textit{Universitat Jaume I}\\
Castellón, Spain \\
jmiravet@uji.es}
\and
\IEEEauthorblockN{Manuel F. Dolz}
\IEEEauthorblockA{\textit{Dept. of Computer Engineering} \\
\textit{Universitat Jaume I}\\
Castellón, Spain \\
dolzm@uji.es}
\and
\IEEEauthorblockN{Jose Duato}
\IEEEauthorblockA{
\textit{Openchip \& Software Technologies}\\
Barcelona, Spain \\
jose.duato@openchip.es}
\and
\IEEEauthorblockN{Enrique S. Quintana-Ortí}
\IEEEauthorblockA{\textit{Dept. of Computer Engineering} \\
\textit{Universitat Politècnica de València}\\
Valencia, Spain \\
enquior@upv.es}
}

\maketitle

\begin{abstract}
Despite the empirical success of \acp{dnn}, their internal training dynamics remain difficult to characterize. In \acs{relu}-based models, the activation pattern induced by a given input determines the piecewise-linear region in which the network behaves affinely. Motivated by this geometry, we investigate whether training exhibits a two-timescale behavior: an early stage with substantial changes in activation patterns and a later stage where weight updates predominantly refine the model within largely stable activation regimes. We first prove a local stability property: outside measure-zero sets of parameters and inputs, sufficiently small parameter perturbations preserve the activation pattern of a fixed input, implying locally affine behavior within activation regions. We then empirically track per-iteration changes in weights and activation patterns across fully-connected and convolutional architectures, as well as Transformer-based models, where activation patterns are recorded in the ReLU feed-forward (MLP/FFN) submodules, using fixed validation subsets. Across the evaluated settings, activation-pattern changes decay 3 times earlier than weight-update magnitudes, showing that late-stage training often proceeds within relatively stable activation regimes. These findings provide a concrete, architecture-agnostic instrument for monitoring training dynamics and motivate further study of decoupled optimization strategies for piecewise-linear networks. For reproducibility, code and experiment configurations will be released upon acceptance.
\end{abstract}

\begin{IEEEkeywords}
Activation patterns, ReLU neural networks, Learning dynamics, Regime change hypothesis, Piecewise-linear geometry, Convergence analysis.
\end{IEEEkeywords}

\section{Introduction}
\ac{dl} has achieved remarkable success over the pasts years, establishing itself as the state-of-the-art paradigm across a wide range of domains, including \ac{cv} and \ac{nlp}, among many others~\cite{lecun2015deep}. This progress builds upon decades of theoretical and algorithmic development, from early connectionist models~\cite{rosenblatt1958perceptron}, through recurrent architectures~\cite{hochreiter1997long}, to convolutional networks for visual recognition~\cite{lecun1998mnist} and to the recent dominance of Transformer-based architectures~\cite{vaswani2017attention}.

\acp{mlp} have played a key role in this progress and constitute a fundamental building block of most deep learning architectures. They are known as universal approximators of continuous nonlinear functions~\cite{hornik1989multilayer}, which is achieved through the composition of linear and nonlinear transformations. When activation functions are piecewise, the network’s expressivity stems from the combinatorial arrangement of activations, which partitions the input space into distinct regions, each exhibiting different functional behavior~\cite{raghu2017expressive}. 

Building on recent insights, we hypothesize that knowledge within \ac{relu} \acsp{mlp} does not evolve uniformly throughout training. Hartmann et al.~\cite{hartmann} analyzed how activation-pattern statistics (e.g., similarity measures) evolve over optimization and how they relate to the network's training dynamics. Along related lines, Duato et al.~\cite{decoupling} argued that it can be useful to distinguish between two complementary aspects of learning in \acs{relu} \acs{mlp}: \textit{structural knowledge}, associated with the activation masks that determine the partition of the input space into piecewise-affine regions, and \textit{quantitative knowledge}, associated with the parameter values that specify the affine mapping within each region. Taken together, these observations motivate the working hypothesis that learning dynamics may admit a decoupled, two-stage interpretation, where activation regimes stabilize earlier while parameters continue to refine within those regimes.

In this work, we provide a local stability result for ReLU activation patterns under small parameter perturbations and complement it with an empirical analysis across architectures and datasets. While local stability does not imply global convergence under \ac{sgd}/Adam, our results motivate a two-stage interpretation of training dynamics in which activation regimes stabilize earlier than weights. The contributions of this paper are threefold:
\begin{enumerate}
 \item \textbf{Local stability of activation patterns in \acs{relu} \acsp{mlp}}. We formalize a measure-theoretic statement showing that, for almost all parameter configurations and inputs, sufficiently small parameter perturbations preserve the activation pattern of a fixed input. This can be extended straightforwardly to piecewise-linear operations on CNNs. 
 \item \textbf{A practical instrumentation of ``activation regime'' dynamics during training}. We introduce a simple per-iteration tracking protocol that compares weight-update magnitudes against activation-pattern changes, enabling side-by-side convergence profiling across models.

 \item \textbf{Empirical evidence across architectures under controlled training}. We report a multi-architecture study on \acp{mlp}, \acp{cnn} and Transformer-based models with \acs{relu} non-linearities showing that activation-pattern changes typically decay earlier than weight-update magnitudes under the evaluated settings, thus motivating a two-stage regime-change viewpoint of training.
\end{enumerate}
The remainder of the paper is organized as follows. Section~\ref{sec:related} reviews related work on network structure, activation patterns, and interpretability. Section~\ref{sec:theory} introduces the definitions and theoretical results underlying our approach. Section~\ref{sec:experiments} presents the experimental evaluation across different configurations. Finally, Section~\ref{sec:conc} concludes the paper and discusses future research directions.

\section{Related Work}
\label{sec:related}
The expressive power of \acs{relu}-based \acsp{dnn} has been a central topic of study over the past decade~\cite{montufar2014number,raghu2017expressive}. One of the main approaches to understanding this connects the success of learning with how the piecewise-linear transformations induced by \acs{relu} activations partition the input space into distinct linear regions~\cite{montufar2014number}. The number and geometry of these regions are bounded by the network’s architecture and provide insight into how data points are clustered and separated in the learned representation space. Although the partition is implicitly governed by activation sign patterns, comparatively fewer works study their dynamics directly. Xu et al.~\cite{xu2024127174} introduced a novel idea of convergence in neural networks, implying that, once the activations are fixed, the neural networks converge as an infinite product of matrices. Hartmann et al.~\cite{hartmann} measured activation statistics such as entropy and similarity in \acs{relu}-based \acsp{dnn}, offering one of the first systematic analyses of activation pattern convergence. Fernández-Hernández et al. ~\cite{fernándezhernández2025ouineedtalkweight} introduced a metric that quantifies how activation patterns across different samples diverge during training, whose behavior correlates with overfitting and underfitting, thus providing a principled way to determine the optimal weight decay value for a given network and dataset. Montavon et al.~\cite{montavon2018methods} proposed an interpretability framework based on studying which features maximize neuron activations, enabling a more fine-grained understanding of feature specialization. 
 Activation patterns also play a role in interpretability, due to its relation to the input space partition. Lingyang et al.~\cite{lingyang} proposed an interpretability method based on the tessellation of the input space into polyhedra, relying on the linearity on each of those for full interpretable methods. Xu et al.~\cite{xu2022traversinglocalpolytopesReLU} study how subtle changes in activation can change the region of interest, providing a sensibility analysis valuable in explainability research.

These studies highlight that activation patterns encode essential information about how neural networks transform and structure their input space, yet the question of how and when these patterns stabilize during training remains open (a gap this work addresses). Duato et al.~\cite{decoupling}, was one of the earliest to propose that activations and weights can be trained in separate regimes, motivating the viewpoint studied here.

\section{Theoretical foundations}
\label{sec:theory}
This section provides the theoretical background necessary to support the results presented in this paper. 
We briefly include the definition of \ac{mlp} used in this work, to unify notation and establish a consistent foundation for the concepts introduced in this section.

\begin{definition}\label{def:mlp}
A \textbf{\acf{mlp}} with $\mathrm{ReLU}$ activations and $L$ hidden layers is a mapping
$f(\,\cdot\,;w):\mathbb{R}^{n_0}\to\mathbb{R}^{n_{L+1}}$ defined as follows.
Let $h_0(x;w)=x$ and, for $l=1,\dots,L$, define the \emph{pre-activations}
\[
z_l(x;w)=W_{l-1}h_{l-1}(x;w)+b_{l-1}\in\mathbb{R}^{n_l},
\]
and the \emph{activations} $h_l(x;w)=\sigma(z_l(x;w))$ with $\sigma(t)=\max(t,0)$ applied elementwise.
The output layer is affine:
\[
f(x;w)=W_L h_L(x;w)+b_L.
\]
\end{definition}

A pivotal definition in this mathematical framework is that of activation pattern, as it will form the basis for the subsequent construction.

\begin{definition}\label{def:activation_pattern}
Let $f(\cdot;w)$ be an MLP as in Definition~\ref{def:mlp}. For a fixed input $x\in\mathbb{R}^{n_0}$,
the neuron $(l,i)$ is \textbf{active} if $z_l(x;w)_i>0$ and \textbf{inactive} otherwise.
The \textbf{activation pattern} is
\(
\mathrm{act}(x;w)=(\mathrm{act}_1(x;w),\dots,\mathrm{act}_L(x;w))\), where \(
\mathrm{act}_l(x;w)=\mathbf{1}\{z_l(x;w)>0\}\in\{0,1\}^{n_l}.
\)
\end{definition}

The activation pattern \(\mathrm{act}(x; w)\) determines the polyhedral region of the input space \(\mathbb{R}^{n_0}\) that contains \(x\). Moreover, \(\mathrm{act}(x; w)\) is not unique to the single sample \(x\): it is shared by all points \(x'\) lying in that same polyhedral region, so every point in the region exhibits the same activation pattern.

Each of these regions corresponds to a convex polyhedron defined by the intersection of half-spaces determined by the ReLU activation states~\cite{montufar2014number}. This geometric structure underlies the expressive power of ReLU-based DNNs~\cite{raghu2017expressive}. While the existence and characterization of these polyhedral regions are well established in the literature, comparatively less attention has been given to understanding the sensitivity and stability of their boundaries; specifically, how small perturbations in parameters or inputs can lead to transitions across adjacent activation regions. 

Here, we formalize this perspective and develop it as a analytical framework to support our subsequent hypotheses on activation pattern dynamics and training stability. Thus, the following proposition establishes that activation patterns in ReLU networks exhibit a form of local stability. Although it is well known that ReLU networks are continuous piecewise-affine and restricted to an affine map in each linear region, we also provide an affine-region corollary to fix notation and make explicit the measure-zero exclusions, which will be assumed in our experiments. To the best of our knowledge, this statement is not commonly spelled out in the measure-theoretic form we use here (with explicit null sets in both parameter and input space); we therefore provide proofs for completeness.

\begin{proposition}[Local stability of activation patterns]\label{theorem:measure_constant}
Let $f(x; w)$ be a ReLU-based MLP on $\mathbb{R}^{n_0}$ with parameters $w \in \mathbb{R}^p$.
Then there exists a set $B \subset \mathbb{R}^p$ of measure zero with the property that
every $w_0 \in \mathbb{R}^p \setminus B$ admits a measure-zero set
$Z \subset \mathbb{R}^{n_0}$ for which the following holds:
for all $x \in \mathbb{R}^{n_0} \setminus Z$ one can find $\varepsilon > 0$ such that
whenever $w \in \mathbb{R}^p$ satisfies $\| w - w_0 \| < \varepsilon$, the activation
pattern of $x$ is identical under $w_0$ and $w$.
\end{proposition}

\begin{proof}
    Let us assume that $w_0 \in \mathbb{R}^p$ satisfies the following condition: for every neuron $(l, i)$, with $l \in \{1, \ldots, L \}$ and $i \in \{1, \ldots, n_l\}$, the corresponding pre-activation function is nonzero on every non-empty open subset of $\mathbb{R}^{n_0}$. In other words, if
    \begin{multline*}
    B = \{w_0 \in \mathbb{R}^p \mid \exists (l,i) \text{ and } \exists \,\mathcal{U} \subseteq \mathbb{R}^{n_0} \\ \text{ open and non-empty s.t. }  z_l(x; w_0)_i = 0 \, \forall x \in \mathcal{U}\},
    \end{multline*}
    then let $w_0 \in \mathbb{R}^p \setminus B$. It turns out that $B$ is a null subset of $\mathbb{R}^p$, although we defer its proof until the end for the sake of clarity.

    For a fixed $w_0$, the function $x \mapsto z_l(x; w_0)_i$ is continuous and piecewise linear in $x$. Hence, there exists a finite collection of closed polyhedra $C_1, \ldots, C_M \subseteq \mathbb{R}^{n_0}$, each defined by affine inequalities ($\le$ or $\ge$), with nonempty and pairwise disjoint interiors, such that $C_1 \cup \ldots \cup C_M = \mathbb{R}^{n_0}$, and for every $x \in C_m$ one has
    \(
    z_l(x; w_0)_i = a_{m,l,i} \cdot x + b_{m,l,i},
    \)
    for some $a_{m,l,i} \in \mathbb{R}^{n_0}$ and $b_{m,l,i} \in \mathbb{R}$ depending on $w_0$.  For each neuron $(l,i)$, let the zero level set of its pre-activation be defined as $Z_{l,i} = \{x \in \mathbb{R}^{n_0} \mid z_l(x; w_0)_i = 0\}$. For every $m \in \{1, \ldots, M\}$, it follows that
    \[
    Z_{l,i} \cap C_m = \{x \in C_m \mid a_{m,l,i} \cdot x + b_{m,l,i} = 0\}.
    \]
    Taking this into account, three cases can be distinguished according to the coefficients of the above equation:
    \begin{enumerate}
        \item If $a_{m,l,i} \ne 0$, the set $Z_{l,i} \cap C_m$ is contained in an $(n_0-1)$-dimensional affine subspace and therefore has measure zero.
        \item If $a_{m,l,i} = 0$ and $b_{m,l,i} \ne 0$, the intersection $Z_{l,i} \cap C_m$ is empty.
        \item If $a_{m,l,i} = 0$ and $b_{m,l,i} = 0$, the equality $Z_{l,i} \cap C_m = C_m$ holds, implying that $\forall x \in \operatorname{int}(C_m)$, $z_l(x; w_0)_i = 0$, with $\operatorname{int}(C_m)$ an open nonempty subset of $\mathbb{R}^{n_0}$, which contradicts the initial assumption.
    \end{enumerate}

    Since $C_1 \cup \ldots \cup C_M = \mathbb{R}^{n_0}$, it follows that $Z_{l,i} = \bigcup_{m=1}^M (Z_{l,i} \cap C_m)$ is a finite union of $(n_0-1)$-dimensional polyhedra and is therefore null in $\mathbb{R}^{n_0}$. Consequently, the set $Z = \bigcup_{l = 0}^{L-1} \bigcup_{i=1}^{n_l} Z_{l,i}$ is also null.

    Let a sample $x \in \mathbb{R}^{n_0} \setminus Z$ be fixed. For each neuron $(l,i)$, the value $z_l(x; w_0)_i$ is nonzero, and hence its sign is either positive or negative. As the mapping $\mathbb{R}^p \rightarrow \mathbb{R}$ given by
    \(    w \longmapsto z_l(x; w)_i
    \)
    is continuous at $w_0$, it follows that there exists $\varepsilon_{l,i} > 0$ such that, for all $w \in \mathbb{R}^p$ satisfying $\|w - w_0\| < \varepsilon_{l,i}$, the same inequality $f_l(x; w)_i > 0$ (or $< 0$) holds. As a result, the pre-activation of each neuron at input $x$ preserves its sign in a neighborhood of $w_0$, and consequently, the corresponding activation state of the neuron $(l,i)$ for the sample $x$ remains unchanged. Let $\varepsilon := \min\{\varepsilon_{l,i} \mid (l,i)\} > 0$. Then, for every $w \in \mathbb{R}^p$ satisfying $\|w - w_0\| < \varepsilon$, the activation pattern of the sample $x$ remains identical under parameters $w$ and $w_0$. This completes the first part of the proof.
    
    It remains to show that
    \begin{multline*}
    B = \{w_0 \in \mathbb{R}^p \mid \exists (l,i) \text{ and } \exists \,\mathcal{U} \subseteq \mathbb{R}^{n_0} \\ \text{ open and non-empty s.t. } z_l(x; w_0)_i = 0 \ \forall x \in \mathcal{U}\}
    \end{multline*}
    is a null subset of $\mathbb{R}^p$. Let us define
    \begin{multline*}
    B_{l,i} = \{w_0 \in \mathbb{R}^p \mid \exists\, \mathcal{U} \subseteq \mathbb{R}^{n_0} \\ \text{ open and non-empty s.t. } z_l(x; w_0)_i = 0 \ \forall x \in \mathcal{U}\}.
    \end{multline*}
    Clearly $B = \bigcup_l \bigcup_i B_{l,i}$, so it suffices to show that each $B_{l,i}$ is null in $\mathbb{R}^p$.  
    Decompose the parameters as $w_0 = (u, \theta)$, where $\theta \in \mathbb{R}^d$ collects the weights and bias of neuron $(l,i)$, and $u \in \mathbb{R}^{p-d}$ corresponds to the remaining parameters. Since the outputs of layer $l-1$ depend only on $u$, they can be written as $\phi(x, u)$.  
    Define $\hat{\phi}(x, u) = (\phi(x, u), 1)$, noting that the last component corresponds to the bias term. Then,
    \(
    f_l(x; w_0)_i = \theta \cdot \hat{\phi}(x, u).
    \) Assume now that $w_0 \in B_{l,i}$. Then there exists an open nonempty set $\mathcal{U} \subseteq \mathbb{R}^{n_0}$ such that
    \[
    0 = z_l(x; w_0)_i = \theta \cdot \hat{\phi}(x, u), \qquad \forall x \in \mathcal{U}.
    \]
    Hence, $\theta$ is orthogonal to the vector space $V(u) = \operatorname{span}\{\hat{\phi}(x, u) \mid x \in \mathcal{U}\}$.  
    Since $\hat{\phi}(x, u) \neq 0$ for all $x \in \mathcal{U}$, it follows that $\dim(V(u)) \ge 1$ in $\mathbb{R}^d$, and thus $\theta \in H(u) := V(u)^\perp$.  
    In fact, $w_0 \in B_{l,i}$ if and only if $\theta \in H(u)$. As $\dim(H(u)) \le d - 1$, the set $\{\theta \mid w_0 \in B_{l,i}\}$ has measure zero.  
    By Fubini’s theorem,
    \[
    \mu(B_{l,i}) = \int_{\mathbb{R}^{p-d}} \mu(\{\theta \mid w_0 \in B_{l,i}\}) \, du
     = 0,
    \]
    and therefore $B_{l,i}$ is null in $\mathbb{R}^p$, as claimed.
\end{proof}

\begin{remark}
The measure-zero exclusions ($w\notin B$, $x\notin Z$) are essential: for fixed $w$, activation changes occur only on a finite union of codimension-one polyhedral facets in $x$ (hence measure zero), so patterns are locally stable for almost every input. Degenerate weights (e.g., $w=0$ in $\mathrm{ReLU}(w\!\cdot\! x)$) make pre-activations vanish on regions with interior, so any perturbation flips the pattern; such singular $w$ form a measure-zero set.
\end{remark}

\begin{corollary}
    \label{prop:local_affine}
Let $f(\,\cdot\,; w)$ be a ReLU-based MLP on $\mathbb{R}^{n_0}$ with parameters
$w \in \mathbb{R}^p$. There exists a measure-zero set $F \subset \mathbb{R}^{n_0}$
such that, for every $x \in \mathbb{R}^{n_0} \setminus F$, one can find
$\varepsilon > 0$, a matrix $K \in \mathbb{R}^{n_{L+1} \times n_0}$, and a vector
$c \in \mathbb{R}^{n_{L+1}}$ satisfying
\begin{equation}
    f(x'; w) = K x' + c, \quad
\forall x' \in \mathbb{R}^{n_0} \text{ s.t. } \|x' - x\| < \varepsilon.
\end{equation}

\end{corollary}
\begin{proof}
    Following the proof of Theorem~\ref{theorem:measure_constant}, fix $w \in \mathbb{R}^p$. There exists a finite collection of closed polyhedra $C_1, \ldots, C_M \subseteq \mathbb{R}^{n_0}$ with nonempty and pairwise disjoint interiors such that $\bigcup_{i} C_i = \mathbb{R}^{n_0}$, where for every $x \in C_m$ one has \(f(x;w)_j=a_{m,j}\cdot x + b_{m,j}\), with $a_{m,j}, b_{m,j}$ dependent on $w$. Let $F=\partial C_1 \cup \dots \cup \partial C_M$, which clearly has measure zero in $\mathbb{R}^{n_0}$. Now, take $x \in \mathbb{R}^{n_0} \setminus F$ such that $x \in C_m$ for some $m$, and therefore $x \in \operatorname{int}(C_m)$. Therefore, exists $\varepsilon > 0$ such that $\forall x'$ satisfying $\|x'-x\|< \varepsilon$ lies in $\operatorname{int}(C_m)$, thus each neuron is described in $x'$ by the same affine expression \(f(x';w)_i = a_{m,i}\cdot x' + b_{m,i}\). Let $\mathrm{act}(x;w)=(\mathrm{act}_1(x;w),\dots,\mathrm{act}_L(x;w))$ denote the activation pattern of $x$ and define
    $D_l=\mathrm{diag}(\mathrm{act}_l(x;w))\in\mathbb{R}^{n_l\times n_l}$ for $l=1,\dots,L$.
    Since $x\in \operatorname{int}(C_m)$, there exists $\varepsilon>0$ such that
    every $x'$ with $\|x'-x\|<\varepsilon$ lies in $\operatorname{int}(C_m)$, therefore $\mathrm{act}(x;w) = \mathrm{act}(x';w)$ and thus induces the same diagonal masks $D_1,\dots,D_L$.
    Therefore, on this neighborhood the network reduces to an affine map
    $f(x';w)=Kx'+c$, where
    \begin{equation*}
    K \;=\; W_L D_L W_{L-1} D_{L-1}\cdots W_1 D_1 W_0
    \end{equation*}
    and
    \begin{equation*}
    c \;=\; b_L \;+\; \sum_{k=0}^{L-1}
    \Big( W_L D_L W_{L-1} D_{L-1}\cdots W_{k+1} D_{k+1} \Big)\, b_k,
    \end{equation*}
    and the result follows. 
\end{proof}

\begin{remark}
The above results extend to convolutional \acs{relu} networks. After a standard linear unfolding of input patches (e.g., \texttt{im2col}), a convolution is represented by multiplication with a structured matrix (Toeplitz/block-Toeplitz), so the activation boundaries remain defined by linear equations and have measure zero. Thus Theorem~\ref{theorem:measure_constant} and Corollary~\ref{prop:local_affine} carry over to standard CNN layers.
\end{remark}

These results support a local ``freezing'' intuition: for almost every input, activation patterns are stable under sufficiently small parameter perturbations. We refer to the empirical manifestation of this effect as the \textbf{regime change hypothesis}: activation regimes tend to stabilize earlier than weights, even while weights continue to adapt within the stabilized regimes (under the evaluated settings). From this point, a \emph{purely \acs{relu}-based} network (e.g., \acsp{mlp} and \acsp{cnn} composed of piecewise-linear operations) can be effectively approximated, in a neighborhood of the data, by a collection of affine operators corresponding to fixed activation regions, with the set of ReLU gates frozen. For Transformers, we track activation masks only in the ReLU FFN/MLP submodules; attention (softmax) lies outside this piecewise-linear analysis. Section~\ref{sec:experiments} quantifies this separation by tracking activation-pattern changes and weight-update magnitudes throughout training.

\section{Experiments}
\label{sec:experiments}

Once the main theoretical definitions and results have been introduced, this section aims to empirically evaluate the claim that, in general, the set of activation patterns stabilizes before the set of network parameters. 

To this end, our goal is not to achieve state-of-the-art performance, but rather to assess the behavior of standard training processes. 
We therefore do not include certain common techniques, such as learning rate scheduling or label smoothing among others, that primarily enhance validation accuracy. Instead, we use simple, schedule-free baselines, but we sweep optimizers/hiperparameters to avoid drawing conclusions from clearly undertrained runs; we report the best-validation configuration per pair. 
Three categories of experiments are performed: \textit{(i)} MLP experiments, evaluating the hypothesis on tabular, sequential, and image datasets; \textit{(ii)} CNN experiments, analyzing similar behavior only on image datasets; and \textit{(iii)} Transformer experiments, examining whether the hypothesis holds when MLPs are embedded within larger architectures. In the Transformer setting, activation patterns are recorded only for the ReLU feed-forward (MLP/FFN) blocks. Every dataset follows the standard train-val split. 

\begin{table}[htpb]
\centering
\caption{MLPs used in the experiments with the configuration of each model specifying the input size (IS), number of classes (NC), hidden dimensions (HD), batch normalization (BN), and dropout (D).}
\label{tab:configs}
\begin{tabular}{lccccc}
\toprule
                    & IS & NC & HD & BN & D \\ \midrule
MLP - Adult         & 104        & 2                 & (256, 128, 64)    & Yes                 & Yes     \\
MLP - Breast Cancer & 30         & 2                 & (64, 32)          & Yes                 & Yes     \\
MLP - Har           & 561        & 6                 & (512, 256)        & Yes                 & No      \\
MLP - MNIST         & 784        & 10                & (256, 128)        & Yes                 & Yes     \\
\bottomrule
\end{tabular}
\end{table}

Within each category, multiple experiments were conducted with different configurations, as summarized below. Each network was trained from scratch, initialized with default settings. 
Early stopping was used in all experiments, with a patience of a third of the total number of epochs and a minimum improvement threshold of $1\%$ for MLPs and CNNs, and $0\%$ for Transformers. 
The training configurations performed are the following.

\begin{itemize}
    \item \textbf{MLPs and LeNet5:} trained for up to 30 epochs using SGD and Adam optimizers, with $\text{learning rate} \in \{10^{-2}, 10^{-3}, 10^{-4}\}$ and $\text{weight decay} \in \{10^{-3}, 10^{-4}\}$. MLPs configurations are provided in Table~\ref{tab:configs}.
    \item \textbf{AlexNet, ResNet34, and EfficientNetB0:} trained for up to 200 epochs using SGD (with and without Nesterov momentum), Adam, and AdamW, with $\text{learning rate} \in \{10^{-2}, 3\cdot10^{-2}, 6\cdot10^{-2}, 10^{-3}, 10^{-4}\}$ and $\text{weight decay} \in \{10^{-3}, 5\cdot10^{-3}, 10^{-4}\}$.
    \item \textbf{Transformers:} trained for up to 50 epochs using SGD, Adam, and AdamW. In most cases, AdamW achieved superior results, as is standard for Transformer architectures. The hyperparameters used were $\text{learning rate} \in \{5\cdot10^{-4}, 3\cdot10^{-4}, 10^{-4}, 2\cdot10^{-5}\}$ and $\text{weight decay} \in \{5\cdot10^{-2}, 10^{-2}\}$.
\end{itemize}

All experiments were conducted using Python~3.10 and PyTorch~2.6.  For all networks, the \acs{relu} activation function was employed, either as specified in the original architecture or by replacing the default activation with ReLU as in EfficientNetB0. For Transformer-based models, the replacement is applied to the feed-forward (MLP/FFN) non-linearities; as specified.

Experiments were performed on a workstation equipped with an NVIDIA~A100 GPU (40~GB~VRAM), a 16-core AMD~EPYC~7282 processor, and 128~GB of system memory. To ensure reproducibility, random seeds were fixed across all relevant libraries.

Table~\ref{tab:experiments_comparison} summarizes the results of our experiments, reporting for each model--dataset pair the optimizer, weight decay (WD), and learning rate (LR) that achieved the best validation score (BVS), together with the epoch at which early stopping was triggered. Training and validation scores correspond to accuracy for all models except GPT-2, for which they reflect perplexity. To quantify \emph{weight convergence}, let $w_i\in\mathbb{R}^{p}$ denote the vector of all trainable parameters after processing the $i$-th training batch. We measure the (normalized) average absolute parameter update between consecutive iterations as
\begin{equation}
\Delta w_i \;=\; \frac{1}{p}\,\lVert w_i - w_{i-1}\rVert_{1}
\;=\; \frac{1}{p}\sum_{j=1}^{p}\bigl|(w_i)_j-(w_{i-1})_j\bigr|,
\label{eq:dw}
\end{equation}
which in practice is computed layerwise by taking elementwise absolute differences of each weight tensor and averaging them. To quantify \emph{activation-pattern convergence}, for each layer $l\in\{1,\dots,L\}$ we use the binary activation mask $\mathrm{act}_l(x;w)\in\{0,1\}^{n_l}$ from Definition~2 and evaluate it on a fixed subset $S$ of validation samples of size $|S|= 0.2\times \text{(batch size)}$, fixed through the whole training process. The per-batch Hamming-type change is computed as
\begin{equation}
\begin{split}
\Delta a_i &= \frac{1}{L}\sum_{l=1}^{L}\Delta a_i^{(l)}, 
\\
\Delta a_i^{(l)} &= \frac{1}{|S|\,n_l}\sum_{x\in S}\bigl\lVert \mathrm{act}_l(x;w_i)-\mathrm{act}_l(x;w_{i-1})\bigr\rVert_{1},
\label{eq:da}
\end{split}
\end{equation}
\textit{i.e.}, the average fraction of bits that flip between consecutive batches (equivalently, the normalized Hamming distance since the masks are binary). These activation masks are computed in eval mode with dropout disabled and BN using running stats. This yields two trajectories, $\{\Delta w_i\}_{i=1}^{T}$ and $\{\Delta a_i\}_{i=1}^{T}$, with $T$ denoting the number of batches. To reduce noise and highlight long-term trends, we smooth each per-batch trajectory using a simple moving average (SMA) implemented as a 1D convolution with a uniform kernel. The smoothing window is set to $30\%$ of the number of batches per epoch. Since weight- and activation-change curves live on different numerical scales, we apply a robust min-max normalization based on percentiles $0.5^{\text{th}}$ and $99.5^{\text{th}}$. This normalization enables within-run trend comparison across scales, so $\mathrm{AUC}$ and $\Delta$ are interpreted as relative cumulative-change proxies rather than direct stabilization times.
Finally, because the batch index is uniformly spaced, we summarize the overall amount of change by the mean value of the normalized curve, $\mathrm{AUC}(x)=\frac{1}{T}\sum_{i=1}^{T}\hat{x}_i$ (equivalently, a discretized area under the curve up to a constant factor). We report $\mathrm{AUC}_W=\mathrm{AUC}(\Delta w)$ for weights and $\mathrm{AUC}_P=\mathrm{AUC}(\Delta a)$ for activation patterns, and define the relative speedup as
$\rho=\mathrm{AUC}_W/\mathrm{AUC}_P$; in this setting, smaller AUC indicates faster stabilization, as it reflects lower cumulative change across training. To support reproducibility, code and full experiment configurations will be released upon acceptance.

As shown in Table~\ref{tab:experiments_comparison}, a consistent trend emerges: for most models, activations exhibit lower cumulative change and typically reach a low-change regime earlier (see threshold analysis/curves) compared with weights, with an average acceleration of $3.85\times$ and a positive and negative standard deviation of $6.43$ and $2.03$, respectively. The only exception is the MLP trained on MNIST, for which activations exhibit slightly slower convergence.

This behaviour can be interpreted by examining the curves shown in Figure~\ref{fig:mlps}, which presents the convergence trajectories for MLP, CNNs and Transformers experiments across the different combinations of model and dataset. For the MLP trained on MNIST, the delayed activation convergence is attributable to substantial fluctuations in activation differences during the initial training batches. In all other experiments, the acceleration is pronounced, typically representing more than a threefold improvement, which highlights the earlier stabilization of activation patterns relative to weights. Overall, the results across all networks are positive. With the exception of MNIST, the MLPs exhibit a clear earlier convergence of activation patterns relative to the weights. This is also observed in the CNNs, where this behaviour is most evident for EfficientNet-B0. The Transformers follow the same trend; implying that this trend appears even if the monitored network is embedded in a larger system.

\begin{table*}[htpb]
\centering
\caption{Comparisons between activation patterns convergence and weight convergence for different models and datasets. Ep. refers to the early stopping epoch. BVS refers to the best validation score of each model -- dataset. $\rho$ refers to the convergence speedup factor between $\mathrm{AUC}_W$ and $\mathrm{AUC}_P$.}
\label{tab:experiments_comparison}
\begin{threeparttable}
\begin{tabular*}{\textwidth}{@{\extracolsep{\fill}}cccccccccc}
\toprule
\multirow{2}{*}{Model} & \multirow{2}{*}{Dataset} & \multirow{2}{*}{Optim.} & \multirow{2}{*}{WD} & \multirow{2}{*}{LR} & \multirow{2}{*}{V.S.} & \multirow{2}{*}{Ep.} & \multirow{2}{*}{$\mathrm{AUC}_W$} & \multirow{2}{*}{$\mathrm{AUC}_P$} & \multirow{2}{*}{$\rho$} \\

                       &                          &                         &                     &                     &                                      &                      &                          &                          &                           \\ \midrule
MLP                    & Adult                    & Adam                    & $10^{-4}$           & $10^{-2}$           & 0.86                                 & 15                   & 0.32                     & 0.30                     & \textbf{1.07}$\times$     \\
MLP                    & Breast Cancer            & SGD                     & $10^{-4}$           & $10^{-2}$           & 0.97                                 & 18                   & 0.32                     & 0.26                     & \textbf{1.23}$\times$     \\
MLP                    & Har                      & SGD                     & $10^{-4}$           & $10^{-2}$           & 0.97                                 & 15                   & 0.20                     & 0.19                     & \textbf{1.05} $\times$    \\
MLP                    & MNIST                    & SGD                     & $10^{-4}$           & $10^{-2}$           & 0.94                                 & 15                   & 0.29                     & 0.41                     & 0.71 $\times$             \\ \doublemidrule
LeNet5                 & MNIST                    & SGD                     & $10^{-4}$           & $10^{-2}$           & 0.99                                 & 25                   & 0.63                     & 0.20                     & \textbf{3.15}$\times$     \\
LeNet5                 & CIFAR10                  & Adam                    & $10^{-4}$           & $10^{-3}$           & 0.64                                 & 29                   & 0.67                     & 0.20                     & \textbf{3.35}$\times$     \\
AlexNet                & CIFAR100                 & SGD                     & $10^{-3}$           & $10^{-3}$           & 0.50                                 & 146                  & 0.80                     & 0.22                     & \textbf{3.6}$\times$      \\
ResNet34               & CIFAR100                 & Adam                    & $10^{-4}$           & $10^{-2}$           & 0.51                                 & 101                  & 0.74                     & 0.66                     & \textbf{1.12}$\times$     \\
ResNet34               & TinyImageNet             & SGD\tnote{a}     & $5\cdot10^{-3}$     & $10^{-3}$           & 0.39                                 & 64                   & 0.63                     & 0.49                     & \textbf{1.3}$\times$      \\
EfficientNetB0         & CIFAR100                 & AdamW                   & $10^{-3}$           & $2\cdot 10^{-2}$    & 0.51                                 & 159                  & 0.58                     & 0.18                     & \textbf{3.2}$\times$      \\
EfficientNetB0         & TinyImageNet             & AdamW                   & $10^{-3}$           & $2\cdot 10^{-2}$    & 0.44                                 & 113                  & 0.42                     & 0.21                     & \textbf{2}$\times$        \\ \doublemidrule
XLNet                  & SST-2                    & AdamW                   & $10^{-2}$           & $2\cdot 10^{-5}$    & 0.83                                 & 4                    & 0.75                     & 0.44                     & \textbf{1.7}$\times$      \\
RoBERTa                & AG News                  & AdamW                   & $10^{-2}$           & $2\cdot 10^{-5}$    & 0.92                                 & 5                    & 0.77                     & 0.22                     & \textbf{3.5}$\times$      \\
T5                     & SQuAD                    & AdamW                   & $10^{-2}$           & $ 10^{-4}$          & 0.93                                 & 6                    & 0.87                     & 0.38                     & \textbf{2.3}$\times$      \\
GPT-2                  & WikiText                 & AdamW                   & $10^{-2}$           & $ 3 \cdot 10^{-4}$  & 24\tnote{b}  & 6                    & 0.97                     & 0.08                     & \textbf{12}$\times$       \\
ViT-B16      & CIFAR100                 & AdamW                   & $5\cdot 10^{-2}$    & $ 10^{-4}$          & 0.61                                 & 50                   & 0.87                     & 0.08                     & \textbf{11}$\times$       \\
Swin                   & Food-101                 & AdamW                   & $5\cdot 10^{-2}$    & $ 10^{-4}$          & 0.72                                 & 50                   & 0.86                     & 0.12                     & \textbf{7.2} $\times$     \\
PyramidViT             & TinyImageNet             & AdamW                   & $5\cdot 10^{-2}$    & $ 5 \cdot 10^{-4}$  & 0.55                                 & 50                   & 0.69                     & 0.07                     & \textbf{9.9}$\times$      \\ \bottomrule
\end{tabular*}

\begin{tablenotes}[flushleft]
\footnotesize
\item[a] Stochastic gradient descent (SGD) with Nesterov momentum set to 0.9.
\item[b] GPT-2 reports perplexity as V.S., not accuracy.
\end{tablenotes}
\end{threeparttable}
\end{table*}

\begin{figure*}[htpb]
\centering
\includegraphics[width=\textwidth]{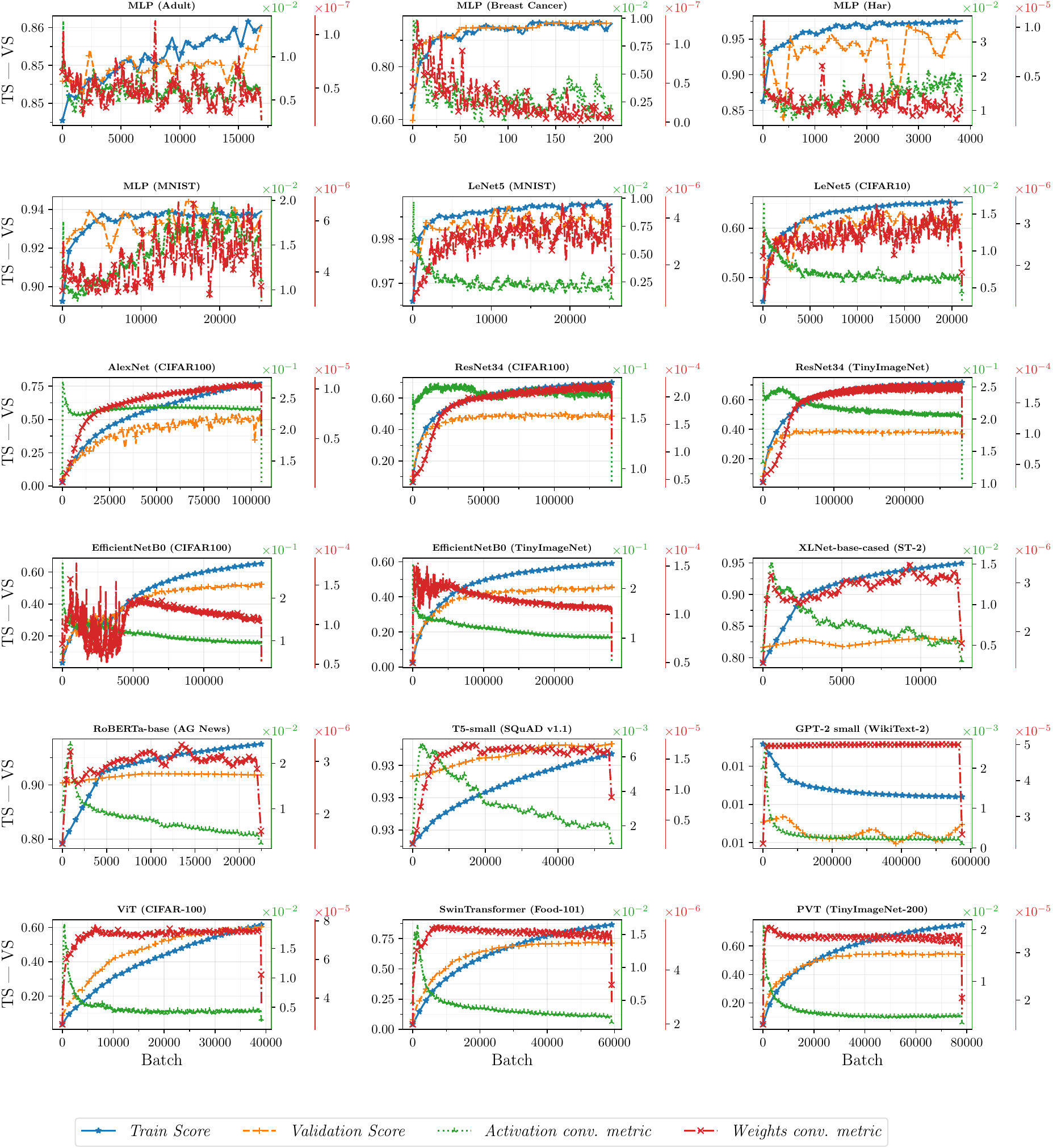}

\caption{Model curves showing the \textcolor{c_train}{training} and \textcolor{c_val}{validation} scores (scale at $y$-left axis), as well as \textcolor{c_hamm}{activation} and \textcolor{c_wdist}{weight} convergences (scale at $y$-right axis), for all model–dataset pairs. The train/validation score (TS/VS) corresponds to train/validation accuracy for all models except GPT-2, for which it reflects perplexity. Activation and weights convergence are measured as described before.}
\label{fig:mlps}
\end{figure*}

\section{Conclusion and Future Work}
\label{sec:conc}
In this work, we have presented both theoretical and empirical results that advance the understanding of convergence dynamics in neural networks, with a particular focus on the relationship between parameter and activation convergence. Concretely, we established that activation patterns remain stable under sufficiently small perturbations of the network parameters. Building on this foundation, we proposed the \textit{regime change hypothesis}, which claims that, during training, activation-pattern changes typically stabilize earlier than weight-update magnitudes (under the evaluated settings). Also, we have provided extensive experimental validation of this hypothesis across a diverse set of architectures, including MLPs, CNNs, and Transformers, consistently supporting our claims. Our theoretical results formalize a local stability property: under well-defined conditions, network weights may continue to evolve while activation patterns only change residually. This observation naturally motivates the regime change hypothesis, which is not formally proven yet the experimental evidence demonstrates that, in the vast majority of training scenarios, activation patterns stabilize significantly earlier than the parameters. The identification of two distinct phases of training (an initial phase in which activation patterns stabilize, followed by a phase of weight refinement with residual changes in the activations) opens several promising directions for future research. This conceptual separation suggests that training could be decomposed into two simpler sub-processes, rather than treated as a single, monolithic optimization task, potentially accelerating convergence and reducing computational costs. This is a promising line of research that remains largely underexplored, with few approximations in the state-of-the-art. It will be interesting to extend the proposed analysis beyond \acs{relu}-based networks. In particular, it remains open how analogous ``regime-change'' behavior manifests in architectures with smooth nonlinearities (e.g., \acs{gelu} or \acs{silu}), and in standard Transformer models where such activations are the default choice. As on going work, we are working on defining comparable activation-regime metrics in these settings and whether similar two-timescale dynamics can be observed; a comprehensive study of non-\acs{relu} networks is therefore left as ongoing work and future research. Moreover, this rezsults carry important implications for interpretability. Once activation patterns are fixed, a purely \acs{relu}-based network can be expressed as an explicit collection of affine transformations over activation regions, each of which is potentially more interpretable and analyzable. In summary, the results presented here establish a theoretical and empirical foundation for viewing neural network training as a two-phase process. By explicitly treating training in this manner, future work may develop more efficient, interpretable, and distributed training strategies, further advancing our understanding of how deep networks learn and generalize.
\section*{Acknowledgment}
This research was funded by the projects PID2023-146569NB-C21 and PID2023-146569NB-C22 supported by MICIU/AEI/10.13039/501100011033 and ERDF/UE. Cristian Pérez-Corral received support from the \textit{Conselleria de Educación, Cultura, Universidades y Empleo} (reference CIACIF/2024/412) through the European Social Fund Plus 2021–2027 (FSE+) program of the \textit{Comunitat Valenciana}. Alberto Fernández-Hernández was supported by the predoctoral grant PREP2023-001826 supported by MICIU/AEI/10.13039/501100011033 and ESF+. Jose I. Mestre was supported by the predoctoral grant ACIF/2021/281 of the \emph{Generalitat Valenciana}. Manuel F. Dolz was supported by the Plan Gen--T grant CIDEXG/2022/13 of the \emph{Generalitat Valenciana}
\bibliographystyle{IEEEtran}

\bibliography{bibliography}

\vspace{12pt}
\clearpage
\end{document}